\documentclass[11pt]{article}
\usepackage{mathrsfs}
\usepackage{amsmath}
\usepackage{amssymb}
\usepackage[english]{babel}
\usepackage[utf8x]{inputenc}
\usepackage[T1]{fontenc}

\usepackage[a4paper,top=3.5cm,bottom=3.5cm,left=3cm,right=3cm,marginparwidth=1.75cm]{geometry}

\usepackage{amsthm}
\usepackage{thmtools, thm-restate} 

\declaretheorem{theorem}
\declaretheorem{lemma}
\declaretheorem{corollary}

\usepackage{graphicx}
\usepackage[colorinlistoftodos]{todonotes}
\renewcommand{\H}{\mathcal{H}}
\newcommand{\K}{\mathcal{K_\mu}}
\newcommand{\Kn}{\mathcal{K_\nu}}
\newcommand{\Kh}{\mathcal{K}_\mu^{1/2}}
\newcommand{\F}{\mathcal{F}}
\newcommand{\I}{\mathcal{R_{\mu}}}
\newcommand{\A}{\mathcal{A}}
\newcommand{\T}{\mathcal{T}}
\newcommand{\LL}{{L^2_\mu}}
\newcommand{\LpL}{{L^p_\mu}}

\renewcommand{\S}{{\mathcal S}}
\newcommand{\bV}{\mathbf{V}}
\newcommand{\R}{\mathbb{R}}
\newcommand{\Br}{\mathcal{B}_R}
\DeclareMathOperator{\Tr}{Tr}
\DeclareMathOperator{\sign}{sign}
\DeclareMathOperator{\supp}{supp}

\newtheorem{innercustomthm}{Theorem}
\newenvironment{customthm}[1]
  {\renewcommand\theinnercustomthm{#1}\innercustomthm}
  {\endinnercustomthm}

\newtheorem{remark}{Remark}

\title{Approximation beats concentration? An approximation view  on inference with smooth radial kernels\footnote{This work appeared in Computational Learning Theory (COLT) 2018.}}
\author{{Mikhail Belkin}\\
Ohio State University,\\ Department of Computer Science and Engineering}

\begin{document}
\setlength{\abovedisplayshortskip}{-.5\baselineskip plus 3pt}
\date{}
\maketitle

\begin{abstract}


Positive definite kernels and their associated Reproducing Kernel Hilbert Spaces provide a mathematically compelling and practically competitive framework for learning from data.

In this paper we take the approximation theory point of view to explore various aspects of  smooth kernels related to their inferential properties. We  analyze eigenvalue decay of  kernels operators and  matrices,  properties of eigenfunctions/eigenvectors and ``Fourier'' coefficients of functions in the kernel space restricted to a discrete set of data points.
We also investigate the fitting capacity of kernels,  giving explicit bounds on the fat shattering dimension of the balls in  Reproducing Kernel Hilbert spaces. 
Interestingly, the same properties that make kernels  very effective approximators for functions in their ``native'' kernel space,  also limit their capacity to represent arbitrary functions.  We discuss various implications, including those for gradient descent type methods.

It is important to note that most of our  bounds are measure independent.  Moreover,  at least in moderate dimension, the bounds for eigenvalues are much tighter than the bounds which can be obtained from the usual matrix concentration results. For example, we see that  eigenvalues of kernel matrices show nearly exponential decay with constants depending only on the kernel and the domain. We call this ``approximation beats concentration'' phenomenon as even when the data are sampled from a probability distribution, some of their aspects are better understood in terms of approximation theory.

%
%
%
%
%



\end{abstract}

\section{Introduction}

Modern supervised machine learning is largely based on Empirical Risk Minimization (ERM), a form of functional approximation. 
Kernel machines perform variants of ERM over Reproducing Kernel Hilbert Spaces (RKHS). RKHS, also known as {\it native spaces} in the approximation literature, are generalizations of Sobolev spaces and have many attractive mathematical and computational properties.  In particular, these spaces
 correspond to positive definite kernels, such as Gaussian, inverse multiquadrics, or Laplace kernels. Inference in these function spaces is analytically tractable, practically competitive  and often leads to  convex optimization problems, which can be viewed as linear methods in infinite dimensional Hilbert spaces. 
 
 
%

In this paper we take a look at the properties of kernels from the approximation point of view. While there is  an extensive and diverse literature on  kernel methods  and their use in machine learning (including the books~\cite{scholkopf2001learning, shawe2004kernel, steinwart2008support}), we are aware of few works that use powerful results now available in the approximation theory literature, with the notable exception of ~\cite{rieger2010sampling}. 
There are a number of learning theory results based on certain assumption about kernel eigenvalue decay. However, there are few analyses showing that specific practically used kernels, e.g., Gaussian kernels, satisfy these assumptions. Moreover, the exact nature of the dependence on the underlying measure has not, to the best of our knowledge, been addressed in the literature. 
We feel that the approximation point of  view provides a rather different perspective on the properties of kernel  methods, their strengths and limitations. 
In particular,  we  show the following:\\
\indent {\bf 1.} Eigenvalues of smooth radial kernel operators/matrices decay at a nearly exponential rate\footnote{By nearly exponential  we will mean a function of the form $O(\exp(-C n^{-\alpha}))$, where $C,\alpha > 0$. } with constants depending only on the kernel and the dimension of the space and independent of the  underlying data sample/measure. This also implies that  kernel matrices/operators corresponding to smooth kernels are uniformly effectively low rank. \\
\indent {\bf 2.} A function in the reproducing kernel space of a smooth kernel $K$ written in the basis of eigenfunctions of $K$, must have nearly exponential coefficient decay in $\LL$ for any measure $\mu$. In particular, $\mu$ can be the empirical measure corresponding to a given dataset. \\
\indent {\bf 3.} Eigenfunctions of a kernel matrix/operator can be nearly exponentially approximated 
by a linear combination of  kernel functions. The span of  eigenfunctions
corresponding to the top eigenvalues is in a  sense
invariant to change of measure.  Significantly,  this is not true for individual eigenvectors, which are strongly influenced by the geometry of the underlying measure. \\
\indent {\bf 4.} The fat shattering ($V_\gamma$) dimension of balls of radius $R$ in the RKHS  of smooth kernels is poly-logarithmic in $\frac{R}{\gamma}$. 
This limits the fitting power of any procedure whose output belongs to a ball of a polynomial size in the RKHS. In particular, this analysis applies to various regularization methods and to gradient descent-like algorithms with  bounded step sizes.\\ 
\indent {\bf 5.} While reducing the width of a kernel (such as a Gaussian) expands the function space, the RKHS corresponding to a wider kernel is contained in the RKHS  of the narrow kernel. 
Thus  combining radial kernels of different bandwidths is unlikely to yield results significantly different from simply using a single kernel with smaller width. 
 
Our results use powerful approximation theory available for radial kernels. At least in moderate dimension they are significantly stronger than   learning theory results not relying on these techniques. In particular,  sample-independent nearly exponential decay for the eigenvalues of kernel matrices seems counter-intuitive in view of the matrix concentration results with rates of $O(\frac{1}{\sqrt{n}})$ (see the discussion in Section~\ref{sec:eigs}).
 
 The paper is organized as follows: in Section~\ref{sec:prelim} we collect  some important background definitions and results on kernels, RKHS and approximation theory. 
In Section~\ref{sec:eigs} we give results on low-rank approximations for kernels and  consequences for  eigenvalue decay of kernel operators/matrices. We proceed to discuss the ``approximation beats concentration'' phenomenon. In Section~\ref{sec:eigenvectors}
we analyze 
 ``Fourier'' coefficient decay of RKHS functions in the basis of eigenvectors. We proceed to show  approximation properties for top eigenvectors and their spans.
 Section~\ref{sec:fat_shatter} gives bounds on the fat shattering dimensions for balls in RKHS and discusses implications for regularized kernel algorithms.
In Section~\ref{sec:width} we address the effect of kernel width. We conclude in Section~\ref{sec:conclusion}.


%
%
%
%
%

\section{Spaces and operators associated to positive definite kernels}
\label{sec:prelim}
We first establish some  background  facts about RKHS, kernels and approximation theory needed for further development. We recommend~\cite{wendland2004scattered} for a comprehensive introduction to the subject.
Let $\Omega$ be a domain\footnote{$\Omega$ can be taken to be a unit cube or a more general bounded domain.} in $\R^d$.    Let $K(x,z)$ be a positive definite kernel on $\R^d$.  We will denote by $\H$ the Reproducing Kernel Hilbert Space (RKHS) corresponding to the kernel $K$.  
Given a probability measure $\mu$  on $\Omega$, we can define the integral operator $\K:\LL \to \LL$: 
\begin{equation}\label{eq:int_operator}
\K f(z)= \int K(x,z)f(x) d\mu
\end{equation}
It is easy to check that $\K$ is a self-adjoint operator on $\LL$ and is compact when the kernel $K(\cdot,\cdot)$ is continuous.  
Notice that while $f \in \LL$ needs to be defined only on the support of the measure $\mu$, Eq.~\ref{eq:int_operator} defines $\K f$  everywhere on $\R^d$. 
We will often consider the case when $\mu$ is supported on a finite set of points, so the difference between the support of $\mu$ and the domain of definition   of $\K f$ is significant. 
Moreover, it can be shown that the $\K f \in \H$ for any $f \in \LL$. Notice also that a function $f \in \H$ gives rise to a function in $\LL$ by restricting it to the support of $\mu$. We will call the restriction operator $\I: \H \to \LL$. 
Thus we will suppress $\I$,  where no ambiguity arises. 
For example, for $f \in \H$, we will write $\|f\|_\LL:=\|\I f\|_\LL$.
Note that $\I$ does not change the function values at any point, just the domain of the definition and the function space norm.  
It can be shown by an extension of the argument in~\cite{wendland2004scattered}  (Proposition 10.28)  that $\I$ is the  adjoint of the kernel operator  $\K:\LL \to \H$. Specifically, for $f \in \H, g\in \LL$ we have
\begin{equation}
\label{eq:adjoint}
\langle f,\K g \rangle_\H = \langle \I f,g\rangle_\LL 
\end{equation}
Moreover, it turns out that  that the square root $\Kh$ exists and is an isometric embedding of $\LL \to \H$. Specifically, for $f,g \in \LL$ we have:
$$
\langle \Kh f,\Kh g \rangle_\H = \langle  f,g\rangle_\LL 
$$

\noindent{\bf Kernel matrices.} Given a set of points $X = \{x_1,\ldots,x_n\} \subset \Omega$ we can construct the corresponding kernel matrix $K_n$, $(K_n)_{ij} = \frac{1}{n}K(x_i,x_j)$. Note that $K_n$ can be viewed as a special case of $\K$, where $\mu$ is a uniform discrete measure on $X$, $\mu = \frac{1}{n}\sum \delta_{x_i}$. \\ 
\noindent{\bf Eigenfunctions/Nystrom extension.}
The eigenfunctions of the  operator $\K$ are defined by the equation 
$$
\int K(x,z)e(x) d\mu_x = \lambda e(z), ~~\lambda \in \R
$$
Note that $e(z)$ is technically an element of $\LL$. However, as the image of $\K$ is actually in $\H$, we  
can define $e(z)$ in $\H$, and, indeed, on all of $\R^d$:
$$
e(z) =\frac{1}{\lambda}\int K(x,z)\,e(x)\, d\mu_x
$$
This formula is known as the {\it Nystrom extension}. 

\begin{remark} Note that technically there are two objects corresponding to  $e(z)$, $e_\H(z) \in \H$ and $e_\LL (z) \in \LL$. These functions coincide on the support of the measure $\mu$, $e_\H(z) = e_\LL(z), z\in \supp(\mu)$,   but have different norms in their respective spaces. 
More precisely (cf., e.g., the discussion in~\cite{rosasco2010learning}),  
$
\I e_\H = e_\LL, ~~~\frac{1}{\lambda}\K e_\LL = e_\H 
$.  Overloading the notation,  we will simply write $e(x)$, while keeping note of the norms. 
\end{remark}
The Nystrom extension allows us to directly compare eigenfunctions of operators with different measures $\mu$,  potentially with disjoint support. We can  compare  eigenvectors of different kernel matrices by comparing the Nystrom extensions of the corresponding operators.  \\
\noindent{\bf The spectral decomposition for kernel operators.} We will now establish some properties of eigenfunctions of kernel operators. Let $e$ and $e'$ be two orthogonal eigenfunctions of $\K$ in $\LL$.  Note that since $\K$ is self-adjoint  any two eigenfunctions corresponding to different eigenvalues are orthogonal. Additionally, 
as $K$ is a positive definite kernel, there are no eigenfunctions			 with eigenvalue $0$. 
Using Eq.\ref{eq:adjoint} We have
\begin{equation}
\langle \K e, \K e' \rangle_\H =   \langle \I \K e,  e'\rangle_\LL = {\lambda} \langle e,  e'\rangle_\LL = 0
\end{equation}
Hence we see that the Nystrom extensions of eigenfunctions orthogonal in $\LL$ are also orthogonal in $\H$. 

The spectral theorem for compact self-adjoint operators (see, e.g.,~\cite{reed1980functional}, Theorem VI.16) guarantees the existence of an orthogonal Hilbert space basis $e_1, e_2, \dots$ of eigenfunctions of $\K$ in $\LL$. 
This basis is finite if $\mu$ is supported on a finite set (and $\LL$ is finite-dimensional) and infinite otherwise.
The discussion above shows that the basis  $e_1, e_2, \dots$ extends to an orthogonal (possibly partial) basis in $\H$.

\noindent {\bf Interpolant operators and the fill}.  Given a set  $X = \{x_1,\ldots,x_n\} \subset \Omega$ and a RKHS $\H$ with a positive definite kernel $K$, we can construct the {\it interpolation operator} $\S_X:\H \to \H$ defined by
$$
\S_X(f) = \arg \min_{g\in \H, ~g(x_i)= f(x_i)} \|g\|_\H 
$$
Setting $K_n$ to be the (positive definite) kernel matrix corresponding to $X$, there is an explicit formula in terms of the inverse of  $K_n$:
\begin{equation}\label{eq:interpolant}
\S_X(f)(x) = \sum \alpha_i K(x_i,x),\,\mathrm{where} ~~ (\alpha_1,\ldots,\alpha_n)^t = K_n^{-1} (f(x_1),\ldots, f(x_n))^t
\end{equation}
From Eq.~\ref{eq:interpolant} it is clear that $\S_X$ is a linear operator, and it can be easily verified that $\S_X(f)(x_i)=f(x_i)$.
In fact, $\S_X$ is an orthogonal  projection operator, which  maps $\H$ to the (finite-dimensional) orthogonal complement to the space of functions vanishing on all points of $X$. 

Another important concept associated to the set $X$, is the {\it fill} $h_X$, which describes how well the set $X$ covers $\Omega$, 
$h_X = \max_{x\in \Omega} \min_{x_i \in X} \|x - x_i\| $.  

\noindent{\bf\it Notation for the norm.} 
Given that we will deal with several functional spaces at once and that ``same'' operators have different norms  depending on the range and the domain, we will use the ``$\to$'' notation. For example, $\| \S_X\|_{\H \to \LpL}: = \sup\limits_{f \in \H,\, f \ne 0} \frac{\|\S_X(f)\|_\LpL}{\|f\|_\H}$ denotes the norm of $\S_X$ as a map from $\H$ to $\LL$. Note that the operator  norm can be defined for any, even {\it non-linear}, map between two normed spaces in the same way.

\noindent{\bf Approximation Theory.}
We will now state the  key result from the approximation theory which provides a bound on the difference $f - \S_X(f)$ in terms of the fill $h_X$. 

Let $K(x,z)$ be a smooth radial kernel.  Specifically, let $K(x,z) = \phi (\|x - y\|)$ and put $f(\cdot):= \phi(\sqrt{\cdot})$. We require that $|f^{(l)}(r)| \le l! M^l$ for all $l$ large enough and $r>0$.  

Two types of important kernels satisfying these conditions are Gaussian kernels $K(x,z)=\exp\left(-{\|x-z\|^2}/{\sigma^2}\right)$ and inverse multiquadric kernels $K(x,z)=(c^2 + \|x-z\|^2)^{-\alpha}, \alpha>0$ (for the popular Cauchy kernel, $\alpha=1$). 

\begin{customthm}{A}\label{thm:approx1} 
Under the conditions on the kernel stated above,  for any set $X\subset \Omega$ and any $p \in [1,\infty]$ there exist constants $C,C' > 0$, such that
$$
\|\I - \S_X\|_{\H \to \LpL} < C' \exp(-C/h_X)
$$
\end{customthm}
Theorem~\ref{thm:approx1} is a special case (and a slight reformulation)  of  Theorem 11.22 in~\cite{wendland2004scattered} (see also~\cite{rieger2010sampling}, Theorem 6.1). This is a powerful approximation theory result showing that any function in $\H$ can be accurately reconstructed from its values at a small number of points. The implications are wide-ranging and, perhaps, surprising in their scope.

\smallskip
\begin{remark}[Gaussian kernels] The bound in Theorem~\ref{thm:approx1}  can be made slightly tighter  for Gaussian kernels, to be of the form $C'\exp (C\log(h_X)/h_X)$ (see~\cite{wendland2004scattered}). The extra logarithmic factor does not substantially change our discussion and we will not treat this case separately. This leads to  slightly looser but more general bounds.
\end{remark}

\section{Low rank approximations to kernel operators and their eigenvalues.}\label{sec:eigs}

We start by showing that any (potentially  non-linear) bounded map $\T$ from a Hilbert space\footnote{A Banach space can be used as well.} to an RKHS $\H$ corresponding to a smooth radial kernel, satisfying the conditions of Theorem~\ref{thm:approx1}, allows a universal  low-rank approximation in $\LpL$. That is,  the output of any such map is close to a low-dimensional subspace in $\H$ according to the norm in $\LpL$. This subspace only depends on $\H$ and is independent of $\mu$ and $\T$.



\begin{theorem}
\label{thm:approx}
Suppose $\T:\bV \to \H$ is a (not necessarily linear) map from a Hilbert (Banach) space $\bV$ to a RKHS of functions on $\R^d$, $\H$. Then there exists a map $\T_n$ from $\bV$ to an $n$-dimensional linear subspace $\H_n\subset \H$, such that  
$$
\|\T - \T_n \|_{\bV \to \LpL} < C' \|\T\|_{\bV \to \H} \exp(-C n^{1/d})
$$ for some constants $C,C'>0$ independent of $\T$ and $\mu$. Moreover:\\ (1) While the map $\T_n$ depends on $\T$, the subspace  $\H_n$ is independent of $\T$.\\ (2) If $\T$ is a linear operator, $\T_n$ is also  a linear operator. 
\end{theorem} 
\begin{proof} 
Let $X=(x_1,\cdots, x_n)$ be a finite subset of $\Omega$. Notice that $\T - \S_X \circ \T = (\I - \S_X) \circ \T$. 
Using the definition of the norm we see that 
$$\| (\I - \S_X) \circ \T \|_{\bV \to \LpL} \le \|\I - \S_X\|_{\H \to \LpL} \| \T\|_{\bV \to \H}.$$
Applying  Theorem~\ref{thm:approx1} we obtain
$
\|\T - \S_X \T\|_{\bV \to \LpL} < C' \|\T\|_{\bV \to \H} \exp(-C/h_X)
$ for some constants $C,C'>0$. 
Choosing the set $X$ appropriately (e.g., a $d$-dimensional grid), we can ensure that $h_X = O(n^{-1/d})$, where $d$ is the dimension of the space. 
Notice that the image of $S_X$ belongs to a $n$-dimensional subspace of $\H$ spanned by the functions  $K(x_i,\cdot), x_i \in X$. Taking $\H_n = \mathrm{span}\{K(x_1,\cdot),\ldots K(x_n,\cdot)\}$ and $\T_n= \S_X \circ \T$ completes the proof. 
\end{proof}
Thus the image of a bounded operator to $\H$ can be nearly exponentially approximated by a finite-dimensional subspace independent of $\T$ and $\mu$.  
To provide a bound on the eigenvalues of kernel operators and matrices we will  need the following perturbation result:
\begin{lemma}
\label{lemma:low_rank}
Suppose $\A$ is a self-adjoint  operator on a Hilbert space $\bV$,  $\A:\bV \to \bV$ and $\A_n$ is a finite rank operator with rank $n$, such that 
$
\|\A-\A_n\| < \epsilon 
$.
Then all eigenvalues of $\A$ except for at most $n$ (counting multiplicity) are smaller than $\epsilon$. 
\end{lemma} 
\begin{proof}
As $\mathrm{rank}~ \A_n=n$, every $n+1$ dimensional subspace contains a non-zero vector $v$, such that $\A_n v = 0$. Suppose $\A$ has at least $n+1$ linearly independent eigenvectors with eigenvalues $\ge \epsilon$. 
Then there exists a vector $v \ne 0$, in the span  of these eigenvectors, such that $\A_n v = 0$. As $\A$ is self-adjoint we can assume that these eigenvectors are orthogonal and hence it is easily seen that $\|\A v\| > \epsilon \|v\|$. We see that $\|\A_n v \| = \|(\A + \A_n - \A)v\| > \|\A v\| -\epsilon \|v\| >0$, which is a contradiction.
\end{proof}

We can now apply Theorem~\ref{thm:approx} to easily obtain a  bound on eigenvalues of kernel matrices and operators. Important related  work includes~\cite{schaback2002approximation,santin2016approximation}, which deal with approximation of spectral properties of integral operators with uniform measure.  
In contrast, we  provide a measure-independent bound, which is key in our learning-theoretic context. 

\begin{theorem}[Eigenvalue decay]\label{th:eigendecay} Let $\kappa = \sup_{x \in \Omega} K(x,x)$. Then  for some $C,C'>0$. 
\begin{equation}\label{eq:eigendecay}
\lambda_i(\K) \le \sqrt{\kappa}\, C'  \exp(-C i^{1/d})
\end{equation}
\end{theorem} 
\begin{proof}
Consider $\K$ as an operator from $\LL \to \H$. Recall from Section~\ref{sec:prelim} that there exists a basis of eigenfunctions of $\K$  in $\LL$ which is also orthogonal in $\H$. Let $e_1$ be an  eigenfunction of $\K$ with the largest eigenvalue. We have
$$
\|\K\|_{\LL \to \H}= \frac{\|\K e_1\|_\H}{\|e_1\|_\LL} = \frac{\lambda_1  \|e_1\|_\H}{\|e_1\|_\LL}
.$$  Recall now that $\K$ is adjoint to the restriction operator $\I: \H \to \LL$. We have $\|e_1\|_\H^2 = \frac{1}{\lambda_1}\langle e_1, \K e_1\rangle_\H = \frac{1}{\lambda_1} \| e_1\|^2_\LL.$ 
Hence 
 $
 \|\K\|_{\LL \to \H} = \frac{\lambda_1}{\sqrt{\lambda_1}}= \sqrt{\lambda_1} \le \sqrt{\Tr(\K)}\le \sqrt{\kappa}
 $.

Applying Theorem~\ref{thm:approx} to $\K$, we have $\|\K - \T_n\|_{\LL \to \LL} < C' \sqrt{\kappa} \exp(-C n^{1/d})$, where $\T_n = \S_X \circ \K$ is a linear operator of rank $n$.  Noticing  that $\K:\LL \to \LL$ is  self-adjoint, we can apply Lemma~\ref{lemma:low_rank}. That completes the proof. 
\end{proof}

\begin{remark}  Notice that all quantities in the inequality in Theorem~\ref{th:eigendecay} are independent of the measure $\mu$. In particular when $\mu$ is a finite measure, $\K$ can be viewed as a matrix.  Hence this result provides a uniform bound on the eigenvalue decay, independent of the size of the matrix. 
\end{remark} 
\noindent{\bf Approximation beats concentration.} 
Suppose $X$ is a set of $n$ points sampled iid from a probability distribution $\mu$ on $\Omega$. Let $\mu_n$ denote the empirical measure associated to $X$. 
Concentration results for matrices (e.g.,~\cite{tropp2015introduction}) combined with  spectral perturbation results  (e.g.,~\cite{rosasco2010learning})   in the bounds for eigenvalues of the form  
$$
|\lambda_{\mu,i} - \lambda_{\mu_n,i}| \le \frac{C}{\sqrt{n}}
$$ 
where $C$ is a constant independent of $i$. 
In comparison, from Eq.~\ref{eq:eigendecay} we see that 
$$
 |\lambda_{\mu,i} - \lambda_{\mu_n,i}| \le \max(\lambda_{\mu,i},\lambda_{\mu_n,i}) \le C' \exp(-C'' i^{1/d})
$$
We see that  approximation ``beats'' concentration by providing a tighter bound as long as $O(\exp(-C'' i^{1/d})) < O(\frac{1}{\sqrt{n}})$.  In other words we need  $n$ to be nearly exponential in the eigenvalue index $i$ for the concentration bounds to be tighter. 
Moreover, unlike concentration results,  approximation also shows that the corresponding eigenvalues must actually be nearly exponentially close to $0$. In addition,  these approximation-based bounds are measure-independent and do not require any iid-type assumption\footnote{On the other hand, strong approximation bounds are specific to the smooth kernel setting while concentration results can be applied to a broad class of random matrix problems. Moreover, unlike approximation, concentration results are often dimension-independent.}.
  
This suggests that significant care should be taken when applying concentration-type analyses of kernel  methods in the iid setting, as  essential inferential properties may become invisible in these statistical analyses. We conjecture that this is one of the reasons why  concentration bounds often turn out to be too pessimistic in practice.

\section{\bf Spectral characterization of RKHS functions and eigenfunctions of kernel operators.}\label{sec:eigenvectors}
\vskip-5pt
We will now use Theorem~\ref{th:eigendecay}  to provide a spectral characterization
 of RKHS functions in terms of their restrictions to finite (or infinite) sets. This characterization should be viewed as parallel to the classical description of Sobolev spaces in terms of their Fourier coefficients.
In particular, we will see that the ``Fourier'' coefficients of a function from $\H$ written in the basis of eigenvectors of  {\it any} kernel matrix, regardless of the dataset, must show nearly exponential decay with coefficient independent of the measure.   
This is significant as in many regression/classification problems we can compute these coefficients explicitly from the data. The decay of the coefficients can thus be analyzed empirically.

\begin{theorem}[Coefficient decay for functions in RKHS]\label{thm:coeff}
Let $f\in \H$ and consider  the restriction of $f$, $\I f \in \LL$. Write the spectral decomposition of $\I f$ in terms of the eigenfunctions $e_i$ of $\K$ as 
$$
\I f = \sum a_i e_i, ~a_i=\langle \I f, e_i\rangle_\LL
$$
Then 
$$
|a_i| \le \sqrt{\lambda_i}\|f\|_\H < 
C'\exp{(-C i^{1/d})} \|f\|_\H
$$
for some $C, C'>0$ independent of $\mu$.
\end{theorem}
\begin{proof}
Recalling Eq.\ref{eq:adjoint}, showing that operator $\K$ is adjoint to $\I$, we have  
$
\langle f,  e_i \rangle_\H=\langle f, \frac{1}{\lambda_i} \K e_i \rangle_\H = \frac{1}{\lambda_i} \langle \I f, e_i\rangle_\LL
$ and 
$$ 
\langle \I f, e_i\rangle_\LL \le \lambda_i  \langle f,  e_i \rangle_\H \le \lambda_i \|e_i\|_\H \|f\|_\H
$$
Notice that 
$$
\|e_i\|_\H^2 = \frac{1}{\lambda_i} \langle e_i, e_i\rangle_\LL = \frac{1}{\lambda_i}
$$
Hence by Theorem~\ref{th:eigendecay}
$$
|a_i|=| \langle \I f, e_i\rangle_\LL|  \le \lambda_i \frac{1}{\sqrt{\lambda_i}} \|f\|_\H = \sqrt{\lambda_i}\|f\|_\H  < {\kappa^{1/4} C'\exp(-C i^{1/d})} \|f\|_\H
$$
for some constants $C,C'>0$  independent of the measure $\mu$.  
\end{proof}

\noindent{\bf Properties of eigenfunctions.} We will now proceed with some basic properties of eigenfunctions which follow from the analysis above. The first observation is that any eigenfunction of  $\K$ with a sufficiently large eigenvalue is well-approximated by the span $K(x_i,\cdot), x_i \in X, i=1,\ldots,n$, where $X$ is chosen as in Theorem~\ref{thm:approx}.  Specifically,
\begin{corollary} Let $\lambda e = \K e$  be an eigenfunction of $\K$. Then there is a function $e_X \in  \mathrm{span}\{K(x_1,\cdot),\ldots K(x_n,\cdot)\}$, such that 
$$
\|e - e_X \|_\LpL  \le  \frac{C}{\lambda}  \exp(-C' n^{1/d})
$$
for some universal constants $C,C'>0$, and any $p \in [1,\infty]$. 
\end{corollary}
\begin{proof}
From Theorem~\ref{thm:approx} we obtain
$$
\frac{1}{\lambda} \|\K e - \S_X \K e \|_\LpL  \le  \frac{C}{\lambda}  \exp(-C' n^{1/d}).
$$ Putting $e_X =  \frac{1}{\lambda} \S_X \K e$ yields the result.
\end{proof}
\begin{remark}
Notice that by taking $p = \infty$ we can make the bound to be pointwise. For example, if $\mu$ is a finite set of ``data'' points,  the approximation holds for  every point uniformly over the choices of $\mu$.  Hence a particular $\mu$ is unimportant in this sense. 
\end{remark}
Similarly, for two measures $\mu$, $\nu$,  the top eigenfunctions/eigenvectors of $\K$ are  nearly contained in the span of the top eigenfunctions of $\Kn$, when restricted to the support of $\nu$. 

\begin{theorem}\label{thm:span_eigs} Let $\lambda_\mu e_\mu = \K e_\mu$  be an eigenfunction of $\K$. Then there is a function $e \in  \mathrm{span}\{e_1,\ldots,e_k\}$, of eigenfunctions of $\Kn$, such that 
for some constants $C,C'>0$
$$
\|e_\nu - e \|_\LL  \le \frac{C}{\sqrt{\lambda_\nu}}~  k^{\frac{d-1}{d}} \exp(-C' k^{-1/d})
$$ 
\end{theorem}
\begin{proof}
First note  that $\|e_\mu\|_\H = \frac{1}{\sqrt{\lambda_\mu}}$. Now write $e_\nu = \sum_{i=1}^\infty a_i e_i$, where $e_i$ are eigefunctions of $\Kn$. Put $e = \sum_{i=1}^k a_i e_i$.
By Theorem~\ref{thm:coeff}, we have for some $C,C'$
$$
\|e_\nu - e\|_\LL^2 \le \sum_{i={k+1}}^\infty a_i^2 \le C'\frac{1}{{\lambda_\mu}} \sum_{i={k+1}}^\infty \exp(-C i^{-1/d}).
$$ 
The last sum can be estimated by noticing that 
$$
\sum_{i={k+1}}^\infty \exp(-C i^{1/d}) < \int_{{k}}^\infty \exp(-C x^{1/d})~dx = d\int_{k^{1/d}}^\infty e^{-C  z} z^{d-1}~d z
$$
Integrating by parts shows that for $k$ sufficiently large, the last integral is of the order $O(k^{(d-1)/d} \exp(-Ck^{1/d}))$, which completes the proof.
\end{proof}
\begin{remark} 
It is interesting to note that 
top eigenvectors of $\K$ contain important information about the  structure of the measure $\mu$, e.g., its clustering structure (e.g.,~\cite{DS_AOS_09}). However, the {\it span} of the top eigenvectors is relatively invariant to the measure. Theorem~\ref{thm:span_eigs} shows that eigenfunctions of $\nu$ will not significantly ``spill'' onto eigenfunctions of $\mu$ with much smaller eigenvalues.
 \end{remark}


\section{The (low) fat shattering dimension of balls in RKHS  and its algorithmic implications}\label{sec:fat_shatter}
\vskip-5pt
Approximation-theoretic results are easily turned into bounds on the fat shattering dimension $V_\gamma$ for balls in RKHS, 
which are significantly tighter than those found in the literature. Combining these bounds with some standard learning theory results, we immediately obtain  generalization guarantees 
for a number of regularized kernel  methods and algorithms including   gradient descent with early stopping. 
%
%
%
We start by recalling the definition of the fat shattering dimension $V_\gamma$ for a function class $\F$ (see, e.g.,~\cite{alon1997scale}).
We say that a set $x_1,\ldots,x_n$ is $\gamma$-shattered by functions from $\F$ if there exist $s_i \in \R, i=1\ldots,n$, such that for any assignments of signs $\sigma_i \in \{-1,1\}$ there is a function $f$ in $\F$, satisfying
\[
 \left.  \begin{array}{l}
    f(x_i) > s_i + \gamma, \textrm{if} ~\sigma_i =1\\
    f(x_i) < s_i - \gamma, \textrm{if} ~\sigma_i =-1\\
  \end{array}\right.
\]
$V_\gamma(\F)$  is taken to be {\it the maximum cardinality} of a  set of points that is $\gamma$-shattered by functions from $\F$. 
To clarify the role of  $s_i$'s, note that $V_0$-dimension\footnote{Note that  0-shatters and ``shatters'' are not exactly the same notion.} is simply the VC-dimension of the sets $\{(x,t) \in \R^d\times\R \,| \,f(x) < t\}$ in $\R^{d+1}$.  Thus $V_\gamma$  for $\gamma>0$ is  a more demanding version of the VC-dimension appropriate for analyzing real-valued functions. 

\begin{theorem}[$V_\gamma$-dimension of RKHS balls]\label{thm:vgamma} Let $\Br:=\{f \in \H,  \|f\|_\H < R\}$ be a ball of radius $R>0$ in $\H$. Under  our standard assumptions on the kernel, we have
\begin{equation}\label{eq:vgamma}
V_\gamma (\Br) < O\left(\log^d\left(\frac{R}{\gamma}\right)\right)
\end{equation}
\end{theorem}
\begin{proof}
Let  $\S$ be a linear space of functions  $\R^d\to \R$. We will say that $\S$   $\gamma$-approximates $\Br$ in the $L^\infty$ norm, if for any $f \in \Br$, there exists $f_1\in \S$, s.t. for any $x$, $|f(x) - f_1(x) | < \gamma$.  Suppose we have  $\S$ that $\gamma$-approximates $\Br$.
By replacing each $f \in \Br$ with its approximation $f' \in \S$, we see that $\S$ will 0-shatter any set of points 
$x_1,\ldots,x_N$ whenever  $\Br$ $\gamma$-shatters  these points.  

Hence $V_\gamma$-dimension of $\Br$ is bounded by the VC-dimension of the subgraph sets  $\{(x,t) \in \R^d\times\R \,| \,f(x) < t\}, f \in \S$. If $\S$ is finite-dimensional, rewriting that inequality in a basis of $\S$, we see that these sets can be viewed in as half-planes in $\R^{\dim \S +1}$,  passing through the origin in that space. It is well-known that VC-dimension of these is simply the dimension of the space which is $\dim \S +1$. Thus we get 
 $
V_\gamma(B_R) \le \dim \S+1
$.
It remains to find a space $\S$, which $\gamma$ approximates $\Br$. From the Theorem~\ref{thm:approx} (with $p=\infty$ and $\bf V = \H$), we obtain
$
C' \|f\|_\H \exp (-Cn^{-1/d}) < \gamma
$.
Solving for $n$ after substituting $\|f\|_H=R$ and taking $\S = \H_n$ completes the proof.
\end{proof}
\begin{remark}
 It should be noted that in contrast to $V_\gamma$ dimension for $\gamma>0$, $V_0$-dimension and, indeed, 
 VC-dimension for the indicator functions $\sign (f), f \in \Br$ are infinite. This follows easily from the interpolating property of $\H$. 
Specifically, for any set $(x_i, y_i), y_i \in \{+1,-1\}$, there exists $f \in \H$, such that $f(x_i) = y_i$.  Scaling this function by a scalar does not change the corresponding indicator functions but allows to make $\|f\|_\H$ arbitrarily small.  
\end{remark}

At this point we should  compare  the bound in Theorem~\ref{thm:vgamma} to  the literature. The paper~\cite{evgeniou1999v}, gives a bound of the form $V_\gamma (\Br) = O\left(\frac{R^2}{\gamma^2}\right)$.  While our bound is generally much tighter, the result from~\cite{evgeniou1999v} is dimension independent and also applies to a broad class of RKHS.  Bounds on the closely related notion of {\it covering numbers} for balls in RKHS corresponding to Gaussian kernels are given in~\cite{steinwart2007fast} (Theorem~3.1). However, the bounds there are still polynomial  in  $\frac{1}{\epsilon}$ (roughly corresponding to our $\gamma$).   The only existing poly-logarithmic result that we are aware of is given in~\cite{zhou2002covering}, where a bound for radial kernels with a slightly worse rate is obtained by using a different approximation theory technique.

Note that an alternative approach to obtain poly-logarithmic bounds for covering numbers similar to those for $V_\gamma$ dimension in Eq.~\ref{eq:vgamma} would be to combine eigenvalue-based capacity bounds from~\cite{guo1999covering} with the eigenvalue bound in our Theorem~\ref{th:eigendecay}.   However, the direct proof is much simpler and, arguably,  more informative. 

\noindent{\bf Generalization bounds.} 
Assume  we are in a standard learning setting where
$(x_i,y_i) \in \Omega\times \{-1,1\}$ is a labeled dataset, $L$ is a Lipshitz loss function and the data are chosen from a probability measure $p$ on $\Omega\times \{-1,1\}$. Suppose our (otherwise arbitrary) learning algorithm $\A$ outputs functions in $\Br$.
In that case our Theorem~\ref{thm:vgamma} together with~\cite{alon1997scale} immediately imply the following ``universal'' generalization bound for smooth radial kernels 
\begin{theorem}[Generalization for kernels] \label{th:generalization}
Let $f$ be the output of our learning algorithm $\A$. Then with high probability
$$
 \frac{1}{n}\sum L(f(x_i),y_i) < E_p L(f(x),y)  +  O\left(\frac{\log^{d/2}{(nR)}}{\sqrt{n}}\right )
$$
\end{theorem}
This bounds applies to most kernel-based learning algorithms as nearly all of them output a function in a certain ball of radius $R$. 
We discuss some of the implications below.\\
\noindent {\bf Algorithmic implications and the limitations of kernel methods}.
We have seen that the   poly-logarithmic bound on the fat shattering dimension in Eq.~\ref{eq:vgamma} implies 
 broad and strong  generalization guarantees given in  Theorem~\ref{th:generalization}. 
The flip side of  that is that even mild regularization (i.e., constraining $R$) imposes severe  limitations on the fitting capacity of kernel methods, at least when the dimension $d$ is not very high\footnote{While $d$  is the ambient dimension,  the effective dimensionality of the  data can be much lower.}.
%

For example, consider the popular ''Tikhonov'' regularizer of the form $\lambda\|f\|_\H^2$. It is easy to see that adding this term in conjunction with a bounded loss function implies that the output of the algorithm  belongs to an RKHS ball with radius $O\left(\frac{1}{\sqrt{\lambda}}\right)$. 
Thus a constant increase to the fitting capacity of $\Br$ requires a nearly exponential increase of $R$. That suggests choosing small values of $\lambda$, which  is consistent with practice, where very small values of $\lambda$ often produce best results\footnote{We note that  even minimum norm interpolation (i.e., $\lambda=0$) shows excellent generalization results~\cite{belkin2018understand}.  While this finding is compatible with our analysis, it is directly explained  by it.}.


Another important example is that of {\it gradient descent} for kernel methods. It is easy to see that in the kernel setting each step of gradient descent increases the norm by 
at most a constant depending only on the kernel and the loss function. Thus, $t$ steps of gradient descent output a function with norm bounded by  $O(t)$. 
As each step of gradient descent typically requires $O(n^2)$ computations (a matrix-vector multiplication), we see that $R=O(c/n^2)$, where $c$ is the number of operations. 
Fixing $\gamma$ (and omitting constants), we see that the dimension of the function space {\it reachable} by $c$ computations is of the order of $\log^{d/2}(\frac{c}{n^2})$. It follows immediately that at least order of $n^2 e^{n^{2/d}}$  computations are needed to fit arbitrary functions on the data points with accuracy $\gamma$.

Note that for square loss, empirical minimization problem is simply matrix inversion, which can be done  using only $n^3$ operations using, e.g., Gaussian elimination.
We see that gradient descent compares unfavorably with matrix inversion in terms of computational complexity, when no assumption about function values are made.  On the other hand, when the function is in RKHS, or more generally, has rapidly decaying coefficients in the ``Fourier basis'' of eigenfunctions, subcubic computational complexity can be demonstrated~\cite{yao2007early, raskutti2014early}. This functional {\it algorithmic reach} for gradient descent with smooth kernels is discussed in~\cite{ma2017diving}  from the spectral decay point of view. Analyzing $V_\gamma$ dimension, as we do here, clarifies that point and connects it  to other standard capacity measures.

 It appears that for many real datasets, gradient descent does require cubic or even super-cubic complexity. Indeed, it could  hardly be expected that nature should co-operate by matching the decay of Fourier coefficients for class membership functions to that  of kernels chosen primarily for computational reasons! 

\begin{remark} 
Our results suggest that smooth kernels would struggle to fit labels assigned randomly to a set of points, as such a fit  would generally require  $O(n^2 e^{n^{2/d}})$ operations (aside from the issues of numerical accuracy).  Indeed,  empirically random assignment  are difficult to fit using smooth kernels, while less smooth  Laplace kernels fit random labels far more easily~\cite{belkin2018understand}. Interestingly, ReLU neural networks appear to be similar to Laplace kernels, capable of fitting random labels with ease~\cite{zhang16understanding}.  
\end{remark}
\begin{remark}
Our bounds  for $V_\gamma$ dimension do not imply that the corresponding VC-dimension is small. Indeed, as noted above, VC-dimension of indicator functions from a ball in RKHS is infinite.  However, a function with a small RKHS norm,
corresponding to a random assignments of labels on a set of data points
must take values which are  exponentially small on all data points.  We conjecture that despite their bounded norm, most of such functions are outside of the computational reach of polynomially many steps of gradient descent. 
\end{remark}

\section{Kernels of different width}\label{sec:width}
\vskip-5pt
We will now briefly discuss the influence of the width (or shape) parameter for kernels from the approximation point of view. 
 It is intuitive that ``narrow'' kernels have better fitting capacity. 
In particular  data can be trivially represented as a sum of $\delta$-functions, which can be thought of as radial kernels of width zero. 
However, it is not apriori clear whether choosing a different kernel width can result in a significantly different function space. We will see  that  making the width of the kernel larger simply shrinks the corresponding RKHS space without adding any new functions. The proof relies on a  quite simple Fourier domain description of RKHS. Despite its usefulness, this characterization   does not seem to be widely known in the learning literature. While we will state the theorem for Gaussian kernels, it also applies to any radial kernel with  Fourier transform that decays fast enough, with the precise condition clear from the proof (see Appendix~\ref{proof:width}).
\begin{theorem}\label{th:width}
 let $K_1(x,z) = \phi(\|x-z\|)$ be a Gaussian kernel  (with $\phi$ a one-dimensional Gaussian) and let 
 $K_2(x,z) = \phi (\|x - z\|/\sigma)$, where 
 $0<\sigma < 1$.  
Let $\H_1$ and $\H_2$ be the corresponding RKHS. Then\\
1.  $\H_1 \subset \H_2$. \\
2. For any $R>0$, the ball of radius $R$ in $\H_1$, $\Br(\H_1)$ is contained in ${\cal B}_{\sigma^{-d/2}R}(\H_2)$  the ball of radius $\frac{1}{\sigma^{d/2}}R$ in  $\H_2$.
On the other hand, $\Br(\H_2) \not\subset \H_1$.
\end{theorem}
\section{Conclusions}\label{sec:conclusion}
\vskip-5pt
The main goal of this note is to bring the powerful tools of approximation theory to kernel learning. 
Approximation theory provides a different perspective on a number of important inferential problems, yielding results which are difficult to obtain using the more standard concentration-based analyses, and are sometimes much tighter.
We have not tried to specify constants  and their dependence on the parameters of the kernel, including the kernel width. This can be done explicitly, and is an important aspect of understanding kernel methods. Furthermore, fundamentally, 
we need to understand how these approximation-based results relate to the intrinsic dimensionality of the data. 
Finally, it is crucial to understand the interface between approximation and concentration.  We believe that combining these modes of analysis, and better understanding the regimes 
where one of them becomes dominant,  can yield significant further insight into kernel inference and, likely, other machine learning problems.\\

\noindent{\bf Acknowledgements.} 
The author is grateful to  Stefanie Jegelka, Siyuan Ma and Daniel Hsu for valuable  discussions on kernel eigenvalues and other topics. He would like to acknowledge  NSF funding and the  hospitality of the Simons Institute for the Theory of Computing, where  early discussions took place.

\bibliographystyle{siam} 
\bibliography{kernel}

\begin{thebibliography}{10}

\bibitem{alon1997scale}
N.~Alon, S.~Ben-David, N.~Cesa-Bianchi, and D.~Haussler.
\newblock Scale-sensitive dimensions, uniform convergence, and learnability.
\newblock {\em Journal of the ACM (JACM)}, 44(4):615--631, 1997.

\bibitem{belkin2018understand}
M.~{Belkin}, S.~{Ma}, and S.~{Mandal}.
\newblock {To understand deep learning we need to understand kernel learning}.
\newblock {\em ArXiv e-prints}, 2018.

\bibitem{evgeniou1999v}
T.~Evgeniou and M.~Pontil.
\newblock On the v-$\gamma$ dimension for regression in reproducing kernel
  hilbert spaces.
\newblock In {\em Algorithmic Learning Theory}, pages 106--117. Springer, 1999.

\bibitem{guo1999covering}
Y.~Guo, P.~L. Bartlett, J.~Shawe-Taylor, and R.~C. Williamson.
\newblock Covering numbers for support vector machines.
\newblock In {\em Proceedings of the twelfth annual conference on Computational
  learning theory}, pages 267--277. ACM, 1999.

\bibitem{ma2017diving}
S.~Ma and M.~Belkin.
\newblock Diving into the shallows: a computational perspective on large-scale
  shallow learning.
\newblock {\em arXiv preprint arXiv:1703.10622, accepted to NIPS 2017}, 2017.

\bibitem{raskutti2014early}
G.~Raskutti, M.~Wainwright, and B.~Yu.
\newblock Early stopping and non-parametric regression: an optimal
  data-dependent stopping rule.
\newblock {\em JMLR}, 15(1):335--366, 2014.

\bibitem{reed1980functional}
M.~Reed and B.~Simon.
\newblock Functional analysis, vol. i, 1980.

\bibitem{rieger2010sampling}
C.~Rieger and B.~Zwicknagl.
\newblock Sampling inequalities for infinitely smooth functions, with
  applications to interpolation and machine learning.
\newblock {\em Advances in Computational Mathematics}, 32(1):103, 2010.

\bibitem{rosasco2010learning}
L.~Rosasco, M.~Belkin, and E.~D. Vito.
\newblock On learning with integral operators.
\newblock {\em Journal of Machine Learning Research}, 11(Feb):905--934, 2010.

\bibitem{santin2016approximation}
G.~Santin and R.~Schaback.
\newblock Approximation of eigenfunctions in kernel-based spaces.
\newblock {\em Advances in Computational Mathematics}, 42(4):973--993, 2016.

\bibitem{schaback2002approximation}
R.~Schaback and H.~Wendland.
\newblock Approximation by positive definite kernels.
\newblock In {\em Advanced Problems in Constructive Approximation}, pages
  203--222. Springer, 2002.

\bibitem{scholkopf2001learning}
B.~Scholkopf and A.~J. Smola.
\newblock {\em Learning with kernels: support vector machines, regularization,
  optimization, and beyond}.
\newblock MIT press, 2001.

\bibitem{shawe2004kernel}
J.~Shawe-Taylor and N.~Cristianini.
\newblock {\em Kernel methods for pattern analysis}.
\newblock Cambridge university press, 2004.

\bibitem{DS_AOS_09}
T.~Shi, M.~Belkin, and B.~Yu.
\newblock Data spectroscopy: eigenspace of convolution operators and
  clustering.
\newblock {\em The Annals of Statistics}, 37, 6B:3960--3984, 2009.

\bibitem{steinwart2008support}
I.~Steinwart and A.~Christmann.
\newblock {\em Support vector machines}.
\newblock Springer Science \& Business Media, 2008.

\bibitem{steinwart2007fast}
I.~Steinwart and C.~Scovel.
\newblock Fast rates for support vector machines using gaussian kernels.
\newblock {\em The Annals of Statistics}, pages 575--607, 2007.

\bibitem{tropp2015introduction}
J.~A. Tropp et~al.
\newblock An introduction to matrix concentration inequalities.
\newblock {\em Foundations and Trends{\textregistered} in Machine Learning},
  8(1-2):1--230, 2015.

\bibitem{wendland2004scattered}
H.~Wendland.
\newblock Scattered data approximation, 2004.

\bibitem{yao2007early}
Y.~Yao, L.~Rosasco, and A.~Caponnetto.
\newblock On early stopping in gradient descent learning.
\newblock {\em Constructive Approx.}, 26(2):289--315, 2007.

\bibitem{zhang16understanding}
C.~Zhang, S.~Bengio, M.~Hardt, B.~Recht, and O.~Vinyals.
\newblock Understanding deep learning requires rethinking generalization.
\newblock {\em CoRR}, abs/1611.03530, 2016.

\bibitem{zhou2002covering}
D.-X. Zhou.
\newblock The covering number in learning theory.
\newblock {\em Journal of Complexity}, 18(3):739--767, 2002.

\end{thebibliography}

\appendix 
\section{Proof of Theorem~\ref{th:width}}\label{proof:width}
The result follows easily  from the Fourier characterization of RKHS, see~\cite{wendland2004scattered}(Theorem 10.12). If $K(x,z)=\psi(x-z)$ is a translation-invariant kernel the RKHS norm of $f$ in the RKHS $\H$ corresponding to $K$ can be written as 
\begin{equation}\label{eq:fourier}
\|f\|_{\H}^2 = (2 \pi )^{-d/2} \int_{\R^d} \frac{|{\cal F}(f)(\omega)|^2}{{\cal F}({\psi})(\omega)} d\omega
\end{equation}
Here ${\cal F}$ denotes the Fourier transform. 
Recall now that (from the scaling property of Fourier transform)  ${\cal F} (\phi (x/\sigma))(\omega) = \sigma^d {\cal F}({\phi})(\sigma \omega)$. 
Since Fourier transform of a Gaussian is also a Gaussian and $\sigma <1$,  we obtain  
$$
\frac{{\cal F} (\phi (x/\sigma))(\omega)}{{\cal F} (\phi) (\omega)} =
\frac{1}{\sigma^d} \frac{{\cal F} (\phi (x))(\sigma\omega)}{ {\cal F}({\phi})( \omega)} \ge \frac{1}{\sigma^d}
$$
Hence using Eq.~\ref{eq:fourier}, we see that
$$
\|f\|_{\H_2} \le \frac{1}{\sigma^{d/2}} \|f\|_{\H_1}  
$$
This completes the proof except for the claim that $\Br(\H_2) \not\subset \H_1$, which is straightforward to see in the Fourier domain.

\end{document}